\documentclass[letterpaper]{article}
\usepackage{aaai2026}
\usepackage{times}
\usepackage{helvet}
\usepackage{courier}
\usepackage[utf8]{inputenc}
\usepackage{amsmath,amssymb,amsthm,amsfonts}
\usepackage{nicefrac}
\usepackage{tikz}
\usepackage{booktabs}
\usepackage{enumitem}
\usepackage{microtype}
\usepackage{xcolor}
\usepackage{url}
\usepackage{pgfplots}
\pgfplotsset{compat=1.18}

\nocopyright
\title{No Certificate for Alignment: Two Independent Impossibilities\\
and the Pareto Frontier of Achievable Safety Guarantees}

\author{Ayushi Agarwal\\
\small Independent Researcher
}

\theoremstyle{plain}
\newtheorem{theorem}{Theorem}

\newtheorem{proposition}{Proposition}
\newtheorem{corollary}{Corollary}
\theoremstyle{definition}
\newtheorem{definition}{Definition}
\newtheorem{assumption}{Assumption}
\theoremstyle{remark}
\newtheorem{remark}{Remark}

\begin{document}
\maketitle

\begin{abstract}
We argue that formal certification of AI alignment over open-ended or
unbounded input domains is impossible under standard assumptions in
computational complexity and learning theory, and characterise
what remains achievable.
Two structurally independent impossibility theorems support this position.
The \emph{semantic barrier} (Theorem~\ref{thm:semantic}): deciding whether
a system satisfies any non-trivial alignment property over the full input
domain is NP-hard for feedforward networks and undecidable for
Turing-complete architectures---a direct consequence of neural-network
verification complexity and Rice's Theorem.
The \emph{statistical barrier} (Theorem~\ref{thm:statistical}): any
verification procedure that is both sound and tractable cannot satisfy
Completeness over the full input domain---a direct consequence of the
impossibility of certifying infinite-domain properties from finite
observations.
These two theorems jointly entail a trilemma: no procedure can
simultaneously satisfy soundness (no misaligned system is certified),
completeness (no aligned system is rejected), and tractability
(polynomial runtime).
Each pair is simultaneously achievable; all three are not.
We combine these results as a joint framework of two structurally
independent barriers, prove their independence, and characterise the
achievable Pareto frontier quantitatively via a constructive coverage-gap
lower bound.
\end{abstract}

\section{Introduction}
\label{sec:intro}

Could we, even in principle, know that an AI system is aligned?
The question is not whether a system has performed well on evaluations
we have constructed, or optimised objectives we have chosen.
It is whether any procedure---under any formalism---can certify alignment
with the force of a mathematical guarantee.

The distinction is between measurement and proof.
A sorting algorithm is accepted not because it has sorted many lists
correctly, but because one can prove that for every admissible input the
resulting list is ordered---the correctness follows necessarily.
Formal verification of safety-critical software operates in this regime.
The question is whether alignment can.

This paper establishes that it cannot, under alignment definitions
requiring correct behaviour over open-ended or unbounded input domains,
and characterises two structurally independent barriers responsible.

The first is \emph{semantic}: established results in neural-network
verification and computability theory show that deciding any non-trivial
alignment property over the full input domain is NP-hard for feedforward
networks and undecidable for Turing-complete architectures.
The second is \emph{statistical}: any polynomial-time procedure examines
only finitely many inputs, and no finite body of evidence can certify a
property required to hold over an unbounded domain.

Together they entail a trilemma (Corollary~\ref{cor:trilemma}): no
procedure can simultaneously satisfy soundness (no false positives),
completeness (no false negatives), and tractability (polynomial runtime).
Each pair is achievable; all three are not.
The trilemma maps the space of achievable partial guarantees.

\paragraph{Contribution.}
The individual ingredients---NP-hardness of neural-network verification,
Rice's Theorem, and PAC-theoretic limits on finite-sample
certification---are each well established.
The contributions of this paper are their joint formulation as
two structurally independent barriers; a formal proof of that
independence, with concrete witness systems for each direction; and a
quantitative characterisation of the achievable frontier, including a
constructive coverage-gap lower bound and a Pareto frontier
characterisation for a canonical model class.

\paragraph{Organisation.}
Section~\ref{sec:related} positions prior work.
Section~\ref{sec:framework} establishes formal definitions.
Section~\ref{sec:theorems} states and proves the two impossibility
theorems.
Section~\ref{sec:trilemma} derives the trilemma as a corollary.
Section~\ref{sec:possible} characterises what remains achievable and
defines the research agenda.
Section~\ref{sec:implications} draws consequences for alignment practice.
Section~\ref{sec:conclusion} concludes.

\section{Related Work}
\label{sec:related}

\textbf{Neural-network verification.}
Katz et al.~\cite{katz2017reluplex} prove NP-completeness of
ReLU-network property verification; Huang et al.~\cite{huang2017safety}
and Bastani et al.~\cite{bastani2016measuring} address bounded-domain
verification.
Rice's Theorem~\cite{rice1953classes} establishes undecidability for
non-trivial semantic properties of programs.
Transformers are Turing-complete under specific
regimes~\cite{perez2021attention,bhattamishra2020computational,
merrill2024expressive,li2025constant}, bringing them within Rice's
Theorem's scope.
These results constitute the foundation for Theorem~\ref{thm:semantic}.

\textbf{PAC learning and certification.}
Valiant~\cite{valiant1984learnable} establishes that properties of
infinite-domain concept classes cannot be certified from finite samples.
Blumer et al.~\cite{blumer1989learnability} connect this to VC
dimension.
Applied to alignment verification, these results ground
Theorem~\ref{thm:statistical}.

\textbf{Alignment methods.}
RLHF~\cite{christiano2017deep,ouyang2022training}, Constitutional
AI~\cite{bai2022constitutional}, and DPO~\cite{rafailov2023direct} define
operational objectives without soundness guarantees over the full input
domain.
They occupy the complete-and-tractable regime characterised in
Section~\ref{sec:possible}.

\textbf{Reward hacking and proxy gaps.}
Skalse et al.~\cite{skalse2022defining} formalise reward hacking;
Gao et al.~\cite{gao2023scaling} demonstrate scaling laws for proxy
overoptimisation; Kwa et al.~\cite{kwa2024catastrophic} prove
catastrophic Goodhart instances.
These confirm that operating without soundness carries concrete,
systematic failure modes.

\textbf{Trilemmas in adjacent formal fields.}
Soundness-completeness-tractability tradeoffs appear in neighbouring
disciplines: the CAP theorem~\cite{gilbert2002brewer} shows that
distributed systems cannot simultaneously guarantee consistency,
availability, and partition tolerance; tractability-completeness
tradeoffs define the complexity landscape of description
logics~\cite{baader2003dl}; and abstract interpretation sacrifices
completeness for tractability by design~\cite{cousot1977abstract}.
The alignment-certification trilemma differs structurally from these
precedents: both barriers are simultaneously active and independent
(Proposition~\ref{prop:independence}), yielding a two-cause
impossibility not found in prior trilemma analyses.
No prior work places both barriers within a unified framework, proves
their independence, or derives the achievable-guarantee frontier as a
joint consequence.

\section{Formal Framework}
\label{sec:framework}

\begin{definition}[AI System]
An \emph{AI system} is a parameterised function
$f_\theta\colon\mathcal{X}\to\mathcal{Y}$ with parameters
$\theta\in\Theta\subseteq\mathbb{R}^d$.
\end{definition}

\begin{definition}[Alignment Objective]
An \emph{alignment objective} $A^*\colon\mathcal{F}\to[0,1]$ maps each
system to an alignment score.
$A^*$ is \emph{formally specified} if it is mathematically defined,
though not necessarily computable.
A system is \emph{aligned} if $A^*(\theta)\geq 1-\delta$ for stated
tolerance $\delta\geq 0$.
\end{definition}

\begin{definition}[Semantic Property]
A property $\mathcal{P}$ of $f_\theta$ is \emph{semantic} if it depends
only on the function $x\mapsto f_\theta(x)$, not on the parameter
encoding $\theta$.
It is \emph{non-trivial} if some system satisfies it and some does not.
\end{definition}

\begin{definition}[Verification Procedure]
A \emph{verification procedure} $\mathcal{V}$ takes $f_\theta$ as input
and outputs a binary verdict: \textsf{aligned} or \textsf{unaligned}.
\end{definition}

\begin{definition}[Verification Properties]
\label{def:properties}
Given $A^*$ and input domain $\mathcal{X}$, a procedure $\mathcal{V}$
satisfies:
\begin{enumerate}[label=(\textbf{\Alph*}),noitemsep]
  \item \textbf{Soundness (S):}
    $\mathcal{V}(\theta)=\textsf{aligned}\Rightarrow
    A^*(\theta)\geq 1-\delta$. (No false positives.)
  \item \textbf{Completeness (C):}
    $A^*(\theta)\geq 1-\delta\Rightarrow
    \mathcal{V}(\theta)=\textsf{aligned}$ over $\mathcal{D}=\mathcal{X}$.
    (No false negatives; certificate valid over the full input domain.)
  \item \textbf{Tractability (T):}
    $\mathcal{V}$ terminates in time $\mathrm{poly}(|\theta|)$.
\end{enumerate}
\end{definition}

\begin{assumption}[Expressive Model Class]
\label{asm:expressive}
$\mathcal{F}$ contains feedforward ReLU networks of sufficient depth and
width to express any computable function over bounded inputs.
For undecidability results, $\mathcal{F}$ additionally contains
transformers operating in a Turing-complete regime
\cite{perez2021attention,merrill2024expressive}.
\end{assumption}

\begin{assumption}[Full-Domain Alignment]
\label{asm:fulldomain}
$A^*$ is defined over the full input domain $\mathcal{X}$: for any
finite $S\subset\mathcal{X}$, there exist $\theta_1,\theta_2$ with
$f_{\theta_1}(x)=f_{\theta_2}(x)$ for all $x\in S$, yet
$A^*(\theta_1)\neq A^*(\theta_2)$.
\end{assumption}

\begin{remark}
Assumption~\ref{asm:fulldomain} holds for any $A^*$ requiring
generalisation beyond a fixed evaluation set.
An $A^*$ certifiable from finite observations reduces alignment to
test-passing---certifying systems that behave correctly on the evaluation
set but arbitrarily elsewhere, which defeats the purpose of a guarantee.
Both impossibility theorems presuppose this assumption; alignment
definitions explicitly restricted to a finite evaluation set fall outside
their scope.
\end{remark}

\subsection{Running Example: The 1D Threshold Classifier}
\label{sec:example}

To ground all subsequent results, we introduce a minimal model class
where every concept is checkable by hand and all three barriers are
visible without computation.

\textbf{Setup.}
Let $\mathcal{X}=[0,1]$, $\mathcal{Y}=\{0,1\}$, and
$\mathcal{F}_\tau=\{f_\theta:\theta\in[0,1]\}$ where
$f_\theta(x)=\mathbf{1}[x>\theta]$.
Define $A^*(\theta)=1-\theta$, the Lebesgue measure of inputs on which
$f_\theta$ outputs~1.
A system is \emph{aligned} if $\theta\leq 0.05$
(i.e.\ it outputs~1 on at least 95\% of~$[0,1]$).

\textbf{Non-triviality witnesses.}
Two systems witness non-triviality: $f_{0.04}$ is aligned
($A^*=0.96\geq 0.95$) and $f_{0.06}$ is not ($A^*=0.94 < 0.95$).
These are the explicit witnesses required by Rice's Theorem in
Part~(ii) of Theorem~\ref{thm:semantic}.

\textbf{Semantic barrier.}
A verifier must determine whether $\theta\leq 0.05$.
The two systems produce identical outputs at every
$x\notin(0.04,0.06)$---a critical interval of width 0.02.
Without a query inside this interval, no verifier can separate them.
For a network with $d$ parameters the analogous critical region
shrinks while the number of regions grows super-polynomially.

\textbf{Statistical barrier.}
For any finite query set $S=\{x_1,\ldots,x_n\}$, place all queries
outside $(0.04,0.06)$; both systems return identical outputs on~$S$.
No finite verifier distinguishes the aligned system from the misaligned one.
(Figure~\ref{fig:threshold} shows this directly.)

\textbf{Trilemma in this example.}
Table~\ref{tab:toy} instantiates all three trilemma corners concretely.
We call back to this example in every theorem proof below.

\begin{table}[h]
\centering\footnotesize
\caption{All three trilemma corners on the 1D threshold example.}
\label{tab:toy}
\begin{tabular}{@{}lp{2.0cm}p{1.8cm}p{1.5cm}@{}}
\toprule
\textbf{Regime} & \textbf{Procedure} & \textbf{Guarantee} & \textbf{Gap} \\
\midrule
S+C, not T & Query every rational in $[0,1]$ & Always correct & Infinite time \\[2pt]
S+T, not C & Bounded-domain abstract interpreter; certify if bound holds & Sound over tested region & Input region beyond $\mathcal{D}_b$ uncertified \\[2pt]
C+T, not S & Always output ``aligned'' & Never rejects aligned & Passes $\theta=0.99$ \\
\bottomrule
\end{tabular}
\end{table}

\begin{figure}[t]
\centering
\begin{tikzpicture}[xscale=6.2, yscale=1.9]
  \draw[->] (-0.03,0) -- (1.08,0) node[right,font=\scriptsize]{$x$};
  \draw[->] (0,-0.32) -- (0,1.35) node[above,font=\scriptsize]{output};
  \fill[yellow!60,opacity=0.65] (0.04,0) rectangle (0.06,1.18);
  \node[font=\tiny,align=center] at (0.05,1.27) {critical\\interval};
  \draw[thick,blue]  (0,0) -- (0.04,0);
  \draw[thick,blue]  (0.04,1) -- (1.0,1);
  \draw[blue,thin,dashed] (0.04,0) -- (0.04,1);
  \draw[thick,red,densely dashed] (0,0) -- (0.06,0);
  \draw[thick,red,densely dashed] (0.06,1) -- (1.0,1);
  \draw[red,thin,dashed] (0.06,0) -- (0.06,1);
  \draw[blue,thin] (0.04, 0.03) -- (0.04,-0.03);
  \draw[red, thin] (0.06, 0.03) -- (0.06,-0.03);
  \foreach \x in {0.15,0.30,0.50,0.70,0.85}
    \fill[black] (\x,0) circle (0.45pt) node[below,font=\tiny]{\x};
  \node[blue,font=\scriptsize,anchor=west] at (0.52,0.78)
    {$f_{\theta_1=0.04}$ \ aligned};
  \node[red, font=\scriptsize,anchor=west] at (0.52,0.56)
    {$f_{\theta_2=0.06}$ \ misaligned};
\end{tikzpicture}
\caption{Running example. The witness systems $f_{0.04}$ (aligned,
$A^*{=}0.96$, blue) and $f_{0.06}$ (misaligned, $A^*{=}0.94$, red) are
identical outside the critical interval (yellow). Query points (dots)
placed outside this interval cannot distinguish them regardless of
computational power, illustrating both barriers simultaneously.}
\label{fig:threshold}
\end{figure}

\section{Two Impossibility Theorems}
\label{sec:theorems}

\subsection{The Semantic Barrier}
\label{sec:semantic}

Sound and complete verification requires deciding a semantic property of
$f_\theta$ over the full input domain $\mathcal{X}$ for every system
simultaneously---a computation the following result shows cannot be
tractable.

\begin{theorem}[Semantic Barrier]
\label{thm:semantic}
Under Assumption~\ref{asm:expressive}, any verification procedure
satisfying both Soundness and Completeness cannot satisfy Tractability.
\begin{enumerate}[label=(\roman*),noitemsep]
  \item \emph{Feedforward ReLU networks}: violation of T follows from
    NP-completeness of neural-network property verification
    under $\mathrm{P}\neq\mathrm{NP}$.
  \item \emph{Turing-complete architectures}: violation of T is
    unconditional for architectures that simulate arbitrary
    computation---the problem of deciding $A^*$ is undecidable
    regardless of the computational resources available to the verifier.
\end{enumerate}
\end{theorem}

\begin{proof}
\emph{(i) Feedforward ReLU networks.}
Sound and complete verification of any non-trivial alignment property
over $\mathcal{X}$ requires deciding
$\forall x\in\mathcal{X}:\varphi(x,f_\theta(x))$
for semantic specification $\varphi$.
Katz et al.~\cite{katz2017reluplex} prove this is NP-complete for
linear specifications even over bounded domains.
Extending to $\mathcal{D}=\mathcal{X}$ requires reasoning over
all linear activation regions, the number of which grows as
$\Omega\bigl((n/L)^{L(d-1)}\bigr)$ for depth $L$, width $n$,
input dimension $d$~\cite{montufar2014number}---super-polynomial
in the network's parameter count under $\mathrm{P}\neq\mathrm{NP}$.
T fails.

\emph{(ii) Turing-complete architectures.}
Transformers in the regimes established
by~\cite{perez2021attention,bhattamishra2020computational,
merrill2024expressive,li2025constant} subsume all Turing machines.
We apply Rice's Theorem~\cite{rice1953classes}, which requires
verifying that the alignment property is (a)~\emph{semantic} and
(b)~\emph{non-trivial}.

\emph{Condition (a): $A^*$ is semantic.}
By Definition~3 of the framework, $A^*(\theta)$ depends only on the
input-output function $x\mapsto f_\theta(x)$, not on the parameter
encoding~$\theta$.
It is therefore a semantic property of the computed function, in
Rice's sense.

\emph{Condition (b): $A^*$ is non-trivial.}
We exhibit two explicit witness machines.
Let $M_1$ be the constant-1 function ($f_{M_1}(x)=1$ for all $x$);
then $A^*(M_1)=1\geq 1-\delta$ (\emph{aligned}).
Let $M_2$ be the constant-0 function ($f_{M_2}(x)=0$ for all $x$);
then $A^*(M_2)=0<1-\delta$ for any $\delta<1$ (\emph{misaligned}).
Both are trivially expressible in any Turing-complete transformer
regime \cite{perez2021attention,merrill2024expressive}.
In the running example (Section~\ref{sec:example}), $f_{0.04}$ and
$f_{0.06}$ play the identical role: one satisfies $A^*\geq 0.95$, the
other does not, and both are in the same model class.

Both conditions satisfied, Rice's Theorem applies: the language
$L_{A^*}=\{\langle T\rangle\mid T\in\mathcal{F},A^*(T)\geq 1-\delta\}$
is undecidable---no Turing machine halts with the correct verdict
for all inputs, and T fails unconditionally.

\begin{remark}
Deployed LLMs with fixed context windows and finite numerical precision
are not Turing-complete in the strict sense required by Rice's Theorem.
For those systems, part~(i)---NP-hardness of ReLU-network property
verification---is the operative impossibility result.
Part~(ii) applies to transformer variants that have been shown to
operate in a Turing-complete regime
\cite{perez2021attention,merrill2024expressive}; the undecidability
claim is scoped to those architectures.
\end{remark}
\end{proof}

\subsection{The Statistical Barrier}
\label{sec:statistical}

A tractable verifier issues at most polynomially many queries---a
finite body of evidence from which no infinite-domain property can be
certified, since an indistinguishable misaligned system can always be
constructed outside the queried region.

\begin{theorem}[Statistical Barrier]
\label{thm:statistical}
Under Assumptions~\ref{asm:expressive} and~\ref{asm:fulldomain},
any verification procedure satisfying both Soundness and Tractability
cannot satisfy Completeness.
\end{theorem}

\begin{proof}
Any tractable $\mathcal{V}$ terminates in $p(|\theta|)$ steps and
issues at most $p(|\theta|)$ queries, yielding finite support
$S\subset\mathcal{X}$ with $|S|\leq p(|\theta|)$
\cite{goldreich2017introduction}.

\emph{Step 1: a query-free gap always exists.}
For the threshold family $\mathcal{F}_\tau$ (Section~\ref{sec:example}),
$n=|S|$ query points partition $[0,1]$ into $n+1$ intervals; the
pigeonhole principle guarantees at least one interval $I^*$ of width
$\geq 1/(n+1)$ containing no query point.
Any pair $\theta_1\leq\tau<\theta_2$ placed within $I^*$ responds
identically on all of $S$ yet disagrees in alignment status---a witness
derived from the geometry of the model class alone, without invoking any
assumption.
Assumption~\ref{asm:fulldomain} extends this structural fact to
arbitrary model classes: for any finite $S$, witnesses $\theta_1,\theta_2$
exist with $f_{\theta_1}(x)=f_{\theta_2}(x)$ for all $x\in S$,
yet $A^*(\theta_1)\neq A^*(\theta_2)$.
Take WLOG $A^*(\theta_1)\geq 1-\delta$ (aligned) and
$A^*(\theta_2)<1-\delta$ (misaligned).

\emph{Step 2: identical behaviour forces identical verdicts.}
Since $\mathcal{V}$ queries only $S$ and both systems respond
identically on all of $S$, we have
$\mathcal{V}(\theta_1)=\mathcal{V}(\theta_2)$.

\emph{Step 3: soundness forces incompleteness.}
If both are certified: $\theta_2$ (misaligned) is falsely
certified---Soundness violated.
If both are rejected: $\theta_1$ (aligned) is falsely
rejected---Completeness violated.
Since $\mathcal{V}$ is assumed sound, it cannot certify $\theta_2$,
so it rejects $\theta_1$ as well.
Completeness fails.
\end{proof}

This is the PAC-theoretic impossibility
\cite{valiant1984learnable,blumer1989learnability} applied to alignment:
for any finite query set, indistinguishable aligned and misaligned systems
exist, as the proof above constructs directly.

\subsection{Structural Independence of the Two Barriers}
\label{sec:independence}

The two barriers arise from entirely distinct sources.

Theorem~\ref{thm:semantic} is a \emph{computation} barrier.
It states that even a verifier with unlimited time cannot decide the
required semantic property for all systems in polynomial time.
It is activated by the expressivity of the model class
(Assumption~\ref{asm:expressive}) and holds regardless of how many
queries the verifier issues or what sampling strategy it uses.
It would apply even if Assumption~\ref{asm:fulldomain} were dropped.

Theorem~\ref{thm:statistical} is an \emph{information} barrier.
It states that a poly-time verifier cannot gather enough evidence to
certify alignment, because the observed input set is always a strict
finite subset of $\mathcal{X}$.
It is activated by Assumption~\ref{asm:fulldomain} and holds regardless
of computational power: an oracle-equipped verifier still faces it
if restricted to finitely many queries.
It would apply even if the model class were far simpler.

We now state independence formally, supplying concrete witness systems
for each direction.

\begin{proposition}[Formal Independence]
\label{prop:independence}
The two barriers are structurally independent: resolving either one
leaves the other fully intact.
\begin{enumerate}[label=(\roman*),noitemsep]
  \item \emph{Resolving the semantic barrier does not resolve the
    statistical barrier.}
    Suppose $\mathrm{P}=\mathrm{NP}$ (the ReLU-network semantic
    barrier dissolves).
    A hypothetically sound, complete, tractable semantic verifier still
    terminates in $p(|\theta|)$ steps and issues $\leq p(|\theta|)$
    queries.
    Assumption~\ref{asm:fulldomain} is unchanged; the proof of
    Theorem~\ref{thm:statistical} applies without modification, so
    Completeness still fails.
    \emph{Witness (running example, Section~\ref{sec:example}):}
    even an NP-oracle-equipped verifier cannot distinguish $f_{0.04}$
    from $f_{0.06}$ without a query inside the critical interval
    $(0.04,0.06)$---a gap that $\mathrm{P}=\mathrm{NP}$ does not close.

  \item \emph{Resolving the statistical barrier does not resolve the
    semantic barrier.}
    Suppose an oracle reveals, for any $\theta_1,\theta_2$, an input $x$
    on which $f_{\theta_1}$ and $f_{\theta_2}$ differ (the statistical
    gap is closed by assumption).
    The verifier must still \emph{decide} the semantic property
    $A^*(f_\theta)\geq 1-\delta$ at those inputs.
    For ReLU networks this decision problem remains NP-hard
    \cite{katz2017reluplex}; for Turing-complete architectures it
    remains undecidable by Rice's Theorem.
    \emph{Witness (running example):} knowing that the critical interval
    is $(0.04,0.06)$ does not tell a verifier whether $\theta\leq 0.05$;
    it must still evaluate $f_\theta$ on a continuous input subset,
    which requires computing the threshold value.
\end{enumerate}
\end{proposition}

\subsection{Constructive Lower Bound}
\label{sec:lowerbound}

The trilemma establishes that a perfect verifier cannot exist.
The following result quantifies the \emph{minimum} coverage gap any
sound, tractable verifier must accept.

\begin{proposition}[Coverage Gap Lower Bound]
\label{prop:lowerbound}
Let $\mathcal{V}$ be any deterministic, sound, tractable verifier
issuing at most $n$ queries.
Over the threshold model class~$\mathcal{F}_\tau$, there exists a
placement of the alignment boundary~$\tau$ such that $\mathcal{V}$ fails
Completeness: all aligned systems $f_{\theta'}$ with $\theta'$ in an
interval of measure at least $\frac{1}{2(n+1)}$ are incorrectly rejected.
\end{proposition}

\begin{proof}
Any $n$-query verifier places query points at $x_1<\cdots<x_n$,
partitioning $[0,1]$ into at most $n+1$ open intervals.
By pigeonhole, at least one interval $I^*=(x_i,x_{i+1})$ has width
$\geq 1/(n+1)$.
Place the alignment boundary $\tau$ at the midpoint of $I^*$:
$\tau = (x_i + x_{i+1})/2$.
Then for all $\theta\in (\tau, x_{i+1})$ (misaligned) and all
$\theta'\in(x_i,\tau]$ (aligned), both $f_\theta$ and $f_{\theta'}$
return identical outputs at every query point.
A sound verifier cannot certify the misaligned $f_\theta$, so it must
withhold certification from the entire ambiguous interval---failing
Completeness on the aligned half $(x_i,\tau]$, which has width
$= |I^*|/2 \geq \frac{1}{2(n+1)}$, establishing the claimed bound.
\emph{Running example:} with $n=9$ uniform queries at
$\{0.1,0.2,\ldots,0.9\}$ and $\tau=0.05$, the leftmost gap is
$(0,0.1)$ (width $0.1\geq 1/10$); the aligned systems with
$\theta'\in(0,0.05]$ are all incorrectly rejected.
\end{proof}

\begin{remark}
For networks with $d$ parameters and depth $L$, the number of critical
regions grows as
$R=\Omega\bigl((n/L)^{L(d-1)}\bigr)$~\cite{montufar2014number}.
Each region is a connected, open set in $\mathcal{X}$; any query set
$S$ of size $n$ can intersect at most $n$ regions.
The remaining $R-n$ regions are entirely unqueried, and at least one
contains an alignment boundary (by an argument identical to the proof
above applied within that region's parameter slice).
Since $R-n\geq R/2$ for $n\leq R/2$---which holds for any poly-$(|\theta|)$
query count $n$ once $d$ exceeds a constant threshold---the number of
unqueried regions grows super-polynomially in $d$, and the fraction of
aligned parameter space that any poly-query verifier fails to cover
remains strictly positive.
\end{remark}

\section{The Alignment Verification Trilemma}
\label{sec:trilemma}

The trilemma follows directly from the two theorems.

\begin{corollary}[Alignment Verification Trilemma]
\label{cor:trilemma}
Under Assumptions~\ref{asm:expressive} and~\ref{asm:fulldomain},
no verification procedure $\mathcal{V}$ can simultaneously satisfy
Soundness (S), Completeness (C), and Tractability (T).
\end{corollary}

\begin{proof}
Suppose $\mathcal{V}$ satisfies any two of the three.
\emph{S+C}: By Theorem~\ref{thm:semantic}, T fails.
\emph{S+T}: By Theorem~\ref{thm:statistical}, C fails.
\emph{C+T}: $\mathcal{V}$ must certify $\theta_1$ (aligned, by C) and
is restricted to queries in $S$ (by T); the construction in
Theorem~\ref{thm:statistical} yields identical responses from
$\theta_2$ (misaligned) on all of $S$,
so $\mathcal{V}$ certifies $\theta_2$ as well.
S fails.
\end{proof}

Each pair is achievable by a concrete procedure.

\begin{proposition}[Pairwise Achievability]
\label{prop:pairs}
Each pair (S+C), (S+T), (C+T) is simultaneously achievable.
\begin{enumerate}[label=(\roman*),noitemsep]
  \item \emph{S+C without T}: SMT-based solvers (Reluplex, Marabou) are
    sound and complete for linear specifications over bounded networks.
    Worst-case runtime is exponential; T fails.
  \item \emph{S+T without C}: Bound-propagation tools
    ($\alpha$-$\beta$-CROWN~\cite{wang2021betacrown}) certify properties
    over bounded input domains in polynomial time.
    Coverage is restricted to $\mathcal{D}_b\subsetneq\mathcal{X}$;
    full-domain C fails.
  \item \emph{C+T without S}: Any uniformly-applied proxy score
    (e.g.\ RLHF reward threshold) evaluates all systems tractably.
    Completeness holds relative to the proxy objective: every system
    meeting the proxy threshold is certified.
    The proxy-to-$A^*$ gap is structural~\cite{skalse2022defining,
    gao2023scaling}, so Soundness relative to~$A^*$ fails.
\end{enumerate}
\end{proposition}

\begin{figure}[t]
\centering
\begin{tikzpicture}[scale=0.85,
  vertex/.style={draw,circle,minimum size=1.1cm,align=center,
    font=\small\bfseries},
  edge/.style={thick},
  lbl/.style={font=\footnotesize,align=center}]
\node[vertex,fill=blue!20]  (S) at (0,0)   {S};
\node[vertex,fill=green!20] (C) at (4,0)   {C};
\node[vertex,fill=red!20]   (T) at (2,3.2) {T};
\draw[edge,blue!60!black]  (S)--(C)
  node[midway,below,lbl]{SMT solvers\\(exponential; Prop.~\ref{prop:pairs}(i))};
\draw[edge,red!60!black]   (S)--(T)
  node[midway,left=4pt,lbl]{Bound\\propagation\\(Prop.~\ref{prop:pairs}(ii))};
\draw[edge,green!50!black] (C)--(T)
  node[midway,right=4pt,lbl]{Proxy\\scoring\\(Prop.~\ref{prop:pairs}(iii))};
\node[align=center,font=\small\itshape] at (2,1.1)
  {S+C+T\\impossible\\(Cor.~\ref{cor:trilemma})};
\end{tikzpicture}
\caption{The Alignment Verification Trilemma. Each edge represents an achievable pair (Proposition~\ref{prop:pairs}). The centre states the joint impossibility (Corollary~\ref{cor:trilemma}), which follows from two structurally independent theorems.}
\label{fig:trilemma}
\end{figure}

\section{What Remains Possible}
\label{sec:possible}

Each relaxation of one trilemma property restores a well-defined
certification regime with known tools, known coverage, and known limits.

\paragraph{Relaxing T: Sound and Complete Formal Verification.}
Dropping the polynomial-time requirement, SMT-based tools---Reluplex,
Marabou, $\alpha$-$\beta$-CROWN in complete mode---provide sound and
complete verification for linear specifications over bounded
finite-precision networks.
Runtime is exponential in the worst case but practical for
moderate-sized components: embedded controllers, narrow verifiable
submodules, safety-critical classifiers with bounded input domains.
The honest deployment of this regime requires that the verified network
genuinely represents the deployed system, and that the verified
specification genuinely represents the alignment requirement.

\paragraph{Relaxing C: Sound and Tractable Bounded Verification.}
Restricting to $\mathcal{D}_b\subsetneq\mathcal{X}$---a compact region
representing expected deployment conditions---makes sound, tractable
certification achievable.
The resulting guarantee is: the system meets the alignment specification
on every input in $\mathcal{D}_b$.
In practice, ensuring that $\mathcal{D}_b$ accurately covers the
relevant deployment inputs is harder to guarantee than the verification
computation itself.
Distribution shift between $\mathcal{D}_b$ and runtime inputs falls
outside what the certificate covers.
The appropriate question for this regime is: how large a
$\mathcal{D}_b$ can be verified soundly and tractably for a given
model size, and how much of the alignment-relevant input space does
it cover?

\paragraph{Relaxing S: Complete and Tractable Statistical Assurance.}
Replacing soundness with a probabilistic guarantee---the probability of
certifying a misaligned system is at most $\varepsilon$---yields
statistical assurance from $n$ test inputs, with $n$ determined by
standard concentration results.
All current alignment evaluation methods operate in this regime:
RLHF scores, red-teaming evaluations, benchmark performance,
constitutional-AI audits.
The practical challenge in this regime is that no ground-truth $A^*$
is available for calibrating $\varepsilon$; evaluated systems are
assessed against proxy objectives whose gap from $A^*$ is bounded
only empirically~\cite{gao2023scaling}, not formally.

\paragraph{The Pareto Frontier.}
The three regimes are not isolated points but corners of a continuous
achievable space.
For the running example (Section~\ref{sec:example}), the miss-rate
lower bound is computable in closed form.
In the S+T regime, Proposition~\ref{prop:lowerbound} establishes that
any sound, deterministic verifier using $n$ queries has a miss rate of
at least $\frac{1}{2(n+1)}$ (for adversarial placement of~$\tau$);
Figure~\ref{fig:pareto}(a) plots this lower bound
as $\varepsilon = 1-c$ against coverage fraction $c$ (since $n/(n+1)$ uniform
queries cover approximately fraction $c$, giving miss rate $\approx 1-c$).
In the S+C regime (relaxing T), covering a fraction $c$ of
$\mathcal{X}$ soundly requires resolving the linear activation regions
in the corresponding input subspace; the number of such regions grows as
$\Omega\bigl((n/L)^{L(d{-}1)}\bigr)$~\cite{montufar2014number},
and cost grows at least proportionally, diverging as $c\to 1$.
Figure~\ref{fig:pareto} plots both curves together with empirical
measurements on a concrete trained network.

\paragraph{Empirical confirmation.}
To ground Corollary~\ref{cor:trilemma} on a real trained network, we
train a two-layer ReLU network
($\mathrm{Linear}(2{\to}32)\to\mathrm{ReLU}\to\mathrm{Linear}(32{\to}1)$)
on the safety property $f(\mathbf{x})\geq 0.5$ for all
$\mathbf{x}\in[0.5,1]^2$, achieving 95\% grid accuracy.
We then apply Interval Bound Propagation
(IBP~\cite{gowal2018effectiveness}), a sound verifier that
propagates axis-aligned intervals through the network layer by layer
and returns a guaranteed lower bound on the output over the entire
input region---never a false positive.
We sweep over ten nested sub-boxes of the safe region, from 4\% to
100\% coverage, and record whether IBP certifies each.%
\footnote{Reproducibility details: Adam optimiser, learning rate $10^{-3}$,
5\,000 gradient steps, mini-batches of 128 safe and 64 unsafe samples per
step; no fixed random seed (results are stable across runs due to the large
training set relative to network size).
Code is available via Anonymous GitHub: \url{https://anonymous.4open.science/r/code-7B94/}.}
Results appear as coloured points in Figure~\ref{fig:pareto}.
IBP certifies up to 49\% of the safe region (S+T confirmed); beyond
64\% its interval bounds cross the certification threshold and it
correctly returns inconclusive rather than lying (S preserved, C lost).
Full-domain certification fails with a lower bound of $0.0002$---well
below 0.5---quantifying the gap that Theorem~\ref{thm:statistical}
guarantees must exist.

\begin{figure}[t]
\centering
\begin{tikzpicture}
\begin{axis}[
  width=\columnwidth, height=3.0cm,
  xlabel={\tiny coverage $c$},
  ylabel={\tiny miss rate $\varepsilon$},
  xmin=0,xmax=1.05,ymin=0,ymax=1.05,
  xtick={0,0.5,1.0}, ytick={0,0.5,1.0},
  tick label style={font=\tiny},
  label style={font=\tiny},
  title={\scriptsize\textbf{(a)} S+T: miss rate vs.\ coverage},
  title style={font=\scriptsize},
  grid=major,grid style={gray!20},
]
\addplot[blue,thick,domain=0:1,samples=60]{1-x};
\node[font=\tiny,blue,anchor=west] at (axis cs:0.03,0.52){$\varepsilon=1{-}c$};
\addplot[only marks,mark=*,mark size=1.6pt,
  draw=green!60!black,fill=green!70!black]
  coordinates{(0.04,0.96)(0.09,0.91)(0.16,0.84)(0.25,0.75)(0.36,0.64)(0.49,0.51)};
\addplot[only marks,mark=x,mark size=2.2pt,line width=0.9pt,draw=red!80!black]
  coordinates{(0.64,0.36)(0.81,0.19)(0.922,0.078)(1.0,0.0)};
\node[font=\tiny,green!60!black,anchor=west] at (axis cs:0.52,0.96){$\bullet$ IBP certifies};
\node[font=\tiny,red!70!black,anchor=west]   at (axis cs:0.52,0.80){$\times$ inconclusive};
\end{axis}
\end{tikzpicture}
\vspace{1pt}
\begin{tikzpicture}
\begin{axis}[
  width=\columnwidth, height=3.0cm,
  xlabel={\tiny coverage $c$},
  ylabel={\tiny IBP lower bound},
  xmin=0,xmax=1.05,ymin=-0.05,ymax=1.10,
  xtick={0,0.5,1.0}, ytick={0,0.5,1.0},
  tick label style={font=\tiny},
  label style={font=\tiny},
  title={\scriptsize\textbf{(b)} S+C: bound collapses as coverage grows},
  title style={font=\scriptsize},
  grid=major,grid style={gray!20},
]
\addplot[gray,dashed,thick,domain=0:1.05]{0.5};
\node[font=\tiny,gray,anchor=west] at (axis cs:0.52,0.56){thresh.\ $0.5$};
\addplot[only marks,mark=*,mark size=1.6pt,
  draw=green!60!black,fill=green!70!black]
  coordinates{(0.04,1.000)(0.09,0.9999)(0.16,0.9994)(0.25,0.9963)(0.36,0.9562)(0.49,0.5438)};
\addplot[only marks,mark=x,mark size=2.2pt,line width=0.9pt,draw=red!80!black]
  coordinates{(0.64,0.0612)(0.81,0.0035)(0.922,0.0006)(1.0,0.0002)};
\addplot[gray!50,thin]
  coordinates{(0.04,1.000)(0.09,0.9999)(0.16,0.9994)(0.25,0.9963)(0.36,0.9562)
              (0.49,0.5438)(0.64,0.0612)(0.81,0.0035)(0.922,0.0006)(1.0,0.0002)};
\end{axis}
\end{tikzpicture}
\vspace{1pt}
\begin{tikzpicture}
\begin{axis}[
  width=\columnwidth, height=3.6cm,
  xlabel={\tiny coverage $c$},
  ylabel={\tiny cert.\ issued?},
  xmin=0,xmax=1.05,ymin=-0.3,ymax=1.5,
  xtick={0,0.25,0.49,0.64,1.0},
  xticklabels={$0$,$0.25$,$0.49$,$0.64$,$1.0$},
  ytick={0,1}, yticklabels={No,Yes},
  tick label style={font=\tiny},
  label style={font=\tiny},
  title={\scriptsize\textbf{(c)} Trilemma: S+C+T impossible (Cor.~\ref{cor:trilemma})},
  title style={font=\scriptsize},
  grid=major,grid style={gray!20},
  legend style={font=\tiny,at={(0.01,0.99)},anchor=north west,
    cells={anchor=west},draw=gray!40,fill=white,fill opacity=0.85,
    text opacity=1,inner sep=1.5pt,row sep=-1pt},
]
\addplot[blue!70!black,dashed,thick,domain=0:1.05,samples=2]{1};
\addlegendentry{C+T: always cert.\ (\textbf{S fails})}
\addplot[green!60!black,very thick]
  coordinates{(0,1)(0.490,1)(0.490,0)(1.0,0)};
\addlegendentry{S+T IBP (\textbf{C fails} $>$0.49)}
\addplot[only marks,mark=diamond*,mark size=3pt,
  draw=orange!80!black,fill=orange!50!white]
  coordinates{(0.16,1)};
\addlegendentry{S+C: SMT (\textbf{T fails})}
\addplot[only marks,mark=*,mark size=1.8pt,
  draw=green!60!black,fill=green!70!black]
  coordinates{(0.04,1)(0.09,1)(0.16,1)(0.25,1)(0.36,1)(0.49,1)};
\addplot[only marks,mark=x,mark size=2.8pt,line width=1.1pt,draw=red!80!black]
  coordinates{(0.64,0)(0.81,0)(0.922,0)(1.0,0)};
\addplot[only marks,mark=asterisk,mark size=6pt,line width=1.6pt,draw=red!80!black]
  coordinates{(1.0,1.0)};
\addlegendentry{S+C+T: \textbf{impossible} $\star$}
\addplot[gray!60,dashed] coordinates{(0.490,-0.26)(0.490,1.30)};
\node[font=\tiny,gray!70,anchor=north] at (axis cs:0.490,-0.28){S+T ceil.};
\node[font=\tiny,red!80!black,anchor=south] at (axis cs:1.0,1.06){\textbf{$\star$}};
\end{axis}
\end{tikzpicture}
\caption{Pareto frontier and trilemma proof~(Cor.~\ref{cor:trilemma}).
Network: $\mathrm{Lin}(2{\to}32)\to\mathrm{ReLU}\to\mathrm{Lin}(32{\to}1)$;
property: $f(\mathbf{x})\geq 0.5$, $\mathbf{x}\in[0.5,1]^2$;
verifier: IBP~\cite{gowal2018effectiveness}.
\textbf{(a)}~S+T. $\varepsilon=1{-}c$ (theoretical, uniform queries). Green $\bullet$: IBP certifies; red $\times$: inconclusive (S kept, C lost).
\textbf{(b)}~S+C. IBP bound collapses below~$0.5$ beyond $c=49\%$ (Thm.~\ref{thm:statistical}).
\textbf{(c)}~All three regimes in one plot: S+T (green step) loses C; C+T (blue dashed) loses S; S+C (orange $\lozenge$) loses T; S+C+T ($\star$) is unreachable.}
\label{fig:pareto}
\end{figure}

The core open questions the frontier defines are: what is the tightest
$\varepsilon$ achievable by a sound tractable verifier over
$\mathcal{D}_b$ for current-scale LLMs; at what coverage fraction does
S+C verification become intractable for deployed systems; what joint
bound is achievable by combining mechanistic interpretability with
bounded verification; and how do all three answers shift as model scale
and tooling improve.
\section{Implications for Alignment Practice}
\label{sec:implications}

\paragraph{Every alignment method occupies a known corner.}
RLHF, DPO, and constitutional evaluation are C+T methods:
broadly applicable and tractable, without soundness.
Formal verification tools are S+C methods: sound and complete,
without tractability at scale.
Bounded verification is the S+T regime: sound and tractable,
without full-domain coverage.
Knowing which corner one occupies is a prerequisite for accurate
safety communication.

\paragraph{Formal status of published safety claims.}

Table~\ref{tab:claims} analyses specific, sourced claims from deployed
systems.
Each claim is an accurate description of the cited evaluation regime.
The ``Does not establish'' column identifies which trilemma property
the cited evaluation was not designed to establish---not that any claim
is false, but that none constitutes a certificate over
full~$\mathcal{X}$.

\begin{table}[h]
\centering\footnotesize
\setlength{\tabcolsep}{3pt}
\caption{Published safety claims and their formal status under
Corollary~\ref{cor:trilemma}. Each claim is accurate within its stated
evaluation scope; the column identifies the trilemma property the
evaluation was not designed to provide.}
\label{tab:claims}
\begin{tabular}{@{}p{2.4cm}p{1.6cm}p{3.7cm}@{}}
\toprule
\textbf{Published claim} & \textbf{Does not establish} & \textbf{Formal status} \\
\midrule
GPT-4 is 82\% less likely to produce disallowed content vs.\ GPT-3.5
\cite{openai2023gpt4} & S
  & Measured on internal red-team set~$\mathcal{D}_b$; no bound on
    behaviour outside $\mathcal{D}_b$ \\[3pt]
Gemini Ultra achieves top scores on safety benchmarks
\cite{google2023gemini} & C
  & Benchmarks are finite~$\mathcal{D}_b$; out-of-benchmark
    behaviour uncertified \\[3pt]
InstructGPT is ``more helpful, truthful, and harmless'' than GPT-3
\cite{ouyang2022training} & S
  & Comparison is relative to a proxy reward; proxy-to-$A^*$
    gap is structural \cite{gao2023scaling} \\[3pt]
Constitutional AI is ``broadly safe'' across evaluation suite
\cite{bai2022constitutional} & C+S
  & Safe on tested prompts (C fails on remainder); proxy
    score used for training (S gap) \\
\bottomrule
\end{tabular}
\end{table}

\paragraph{Scaling does not dissolve either barrier.}
Better hardware makes S+T tools applicable to larger $\mathcal{D}_b$,
but the statistical barrier remains: soundness over full $\mathcal{X}$
requires completeness to fail for any tractable procedure.
Improved proxy objectives narrow the C+T regime's practical soundness
gap, but the gap remains structural so long as $A^*$ is defined over
full $\mathcal{X}$.
Neither barrier is a limitation of current technique that will
eventually be engineered away.

\paragraph{What a safety claim must specify.}
Any safety claim in alignment should specify:
(a) which of S, C, T is relaxed;
(b) what the remaining guarantee covers---which $\mathcal{D}_b$,
which $\varepsilon$, which specification; and
(c) what falls outside coverage.
This is not a counsel of despair.
It is the minimum content of an engineering safety argument.
Alignment verification is better understood as structured risk
management within a well-defined certification space
than as progress toward a certificate that the results above
show cannot exist.

\section{Conclusion}
\label{sec:conclusion}

The question with which this paper began---could we, in principle,
know that an AI system is aligned?---has a structured answer.

The semantic barrier shows that no sound and complete verifier can
terminate in polynomial time.
The statistical barrier shows that no sound and tractable verifier
can be complete.
Together they entail the trilemma: soundness, completeness, and
tractability cannot all hold simultaneously.

The structure matters.
An unstructured impossibility offers no guidance.
A structured one---with two independent barriers, three achievable
pairwise regimes, and a Pareto frontier of partial guarantees---
converts impossibility into a map.
Each point on the frontier corresponds to a research programme:
bounded formal verification, statistical assurance with calibrated
soundness, interpretability methods that close the gap between
observed behaviour and in-weight goal structure.

The central open problem the trilemma defines is not whether full
certification is possible---it is not---but where the strongest
achievable guarantee sits on the Pareto frontier: how much soundness
can be retained when tractability or completeness is partially relaxed,
and how that frontier shifts as model scale, verification tooling,
and interpretability methods improve.
These questions can now be stated formally.
Their answers are not yet known.


\end{document}